\documentclass[letterpaper, 10 pt, journal, twoside]{IEEEtran}
\IEEEoverridecommandlockouts

\usepackage{amsmath,amsthm} 
\usepackage{amssymb} 
\usepackage{graphicx}
\usepackage{calc}
\usepackage{mathtools}

\usepackage[noadjust]{cite}

\usepackage{subcaption}
\usepackage[font=footnotesize]{caption}
\usepackage{xcolor}
\usepackage{url}
\usepackage{booktabs}

\usepackage{enumitem,kantlipsum}
\usepackage{float}
\usepackage{textcomp}
\usepackage[absolute]{textpos}
\usepackage{nicefrac,xfrac}
\usepackage{graphics}
\usepackage{multirow}
\theoremstyle{definition}
\newtheorem{problem}{Problem}
\newtheorem{theorem}{Theorem}
\newtheorem{assumption}{Assumption}

\newtheorem{observation}{Observation}

\newtheorem*{remark*}{Remark}

\newtheorem{lemma}{Lemma}

\renewcommand{\baselinestretch}{1}
\usepackage{bm}

\newcommand{\be}{\begin{equation}}
\newcommand{\bs}{\begin{split}}
\newcommand{\es}{\end{split}}
\newcommand{\ee}{\end{equation}}

\newcommand{\T}{\mathcal{T}}
\newcommand{\E}{\mathbb{E}}
\newcommand{\TSP}{\pi^{\text{TSP}}}

\newcommand{\xoverbrace}[2][\vphantom{\dfrac{A}{A}}]{\overbrace{#1#2}}
\newcommand{\Batch}{\mathtt{BATCH}}
\newcommand{\DC}{\mathtt{DC}}
\newcommand{\Cont}{\mathtt{EVENT}}

\newcommand{\rpm}{\raisebox{.2ex}{$\scriptstyle\pm$}}

\usepackage{algorithm}

\newcommand{\blue}[1]{\textcolor{blue}{#1}}

\newcommand{\Told}{\mathcal{T}^{\text{old}}}

\usepackage[noend]{algpseudocode}

\newsavebox\myboxA
\newsavebox\myboxB
\newlength\mylenA

\newcommand*\xoverline[2][0.75]{%
 \sbox{\myboxA}{$\m@th#2$}%
 \setbox\myboxB\null
 \ht\myboxB=\ht\myboxA%
 \dp\myboxB=\dp\myboxA%
 \wd\myboxB=#1\wd\myboxA
 \sbox\myboxB{$\m@th\overline{\copy\myboxB}$}
 \setlength\mylenA{\the\wd\myboxA}
 \addtolength\mylenA{-\the\wd\myboxB}%
 \ifdim\wd\myboxB<\wd\myboxA%
  \rlap{\hskip 0.5\mylenA\usebox\myboxB}{\usebox\myboxA}%
 \else
  \hskip -0.5\mylenA\rlap{\usebox\myboxA}{\hskip 0.5\mylenA\usebox\myboxB}%
 \fi}
\makeatother

\title{
   Optimizing Task Waiting Times \\ 
   in  Dynamic Vehicle Routing
}

\author{
    Alexander Botros, Barry Gilhuly, Nils Wilde, Armin Sadeghi, Javier Alonso-Mora, and Stephen L.\ Smith
    \thanks{This research is partially supported by the Natural Sciences and Engineering Research Council of Canada (NSERC) and by the European Union's Horizon 2020 research and innovation program under Grant 101017008.}
    \thanks{A.\ Botros, B.\ Gilhuly, A.\ Sadeghi, and S.\ L.\ Smith are with the Department of Electrical and Computer Engineering, University of Waterloo, Canada. N.\ Wilde and J.\ Alonso-Mora are with Delft University of Technology.
    \texttt{\{n.wilde,j.alonsomora\}@tudelft.nl}
    \texttt{\{abotros,a6sadegh,bgilhuly,stephen.smith\}}\quad\texttt{@uwaterloo.ca}}
    \thanks{Code available at~\protect\url{https://github.com/arminsadeghi/mDVRP-optimal-policy}}
}

\TPMargin{5pt}

\IEEEpubid{\begin{minipage}{\textwidth}\ \\[12pt] \centering
		\footnotesize \copyright{}2023 IEEE. Personal use of this material is permitted.  Permission from IEEE must be obtained for all other uses, in any current or future media, including reprinting/republishing this material for advertising or promotional purposes, creating new collective works, for resale or redistribution to servers or lists, or reuse of any copyrighted component of this work in other works.
\end{minipage}} 

\begin{document}
\maketitle

\begin{abstract}
We study the problem of deploying a fleet of mobile robots to service tasks that arrive stochastically over time and at random locations in an environment. This is known as the Dynamic Vehicle Routing Problem (DVRP) and requires robots to allocate incoming tasks among themselves and find an optimal sequence for each robot.
State-of-the-art approaches only consider average wait times and focus on high-load scenarios where the arrival rate of tasks approaches the limit of what can be handled by the robots while keeping the queue of unserviced tasks bounded, \textit{i.e.,} stable. To ensure stability, these approaches repeatedly compute minimum distance tours over a set of newly arrived tasks. This paper is aimed at addressing the missing policies for moderate-load scenarios, where quality of service can be improved by prioritizing long-waiting tasks.
We introduce a novel DVRP policy based on a cost function that takes the $p$-norm over accumulated wait times and show it guarantees stability even in high-load scenarios. We demonstrate that the proposed policy outperforms the state-of-the-art in both mean and $95^{\text{th}}$ percentile wait times in moderate-load scenarios through simulation experiments in the Euclidean plane as well as using real-world data for city scale service requests.
\end{abstract}

\begin{IEEEkeywords}
Path Planning for Multiple Mobile Robots or Agents; Planning, Scheduling and Coordination; Task Planning
\end{IEEEkeywords}

\IEEEpubidadjcol

\section{Introduction}
The Dynamic Vehicle Routing Problem is an important and long-studied challenge in assigning autonomous vehicles (or robots) to tasks as they appear in the environment. Applications include pickup-and-delivery \cite{grippa2019drone, gao2021capacitated}, mobility-on-demand transportation systems~\cite{zardini2022analysis, zhang2016control, alonso2017demand}, sensor networks \cite{bullo2011dynamic, pavone2010adaptive}, surveillance \cite{ozkan2021uav, liu2019optimization}, personal care \cite{sadeghi2018re, wilde2022online}, and environmental monitoring~\cite{chandarana2021planning, sousa2020decentralized,smith2011persistent}.

In a typical application, one or more robots assign and schedule incoming tasks to optimize their quality of service. When the set of tasks is finite and known \emph{a priori}, the resulting problem is known as \emph{Vehicle Routing}: a fleet of $m$ robots services $n$ tasks while maximizing some service measure. When new tasks arrive over time, the problem becomes the \emph{dynamic vehicle routing problem} (DVRP).
The conventional approach for multiple-robot vehicle routing with provable properties~\cite{bullo2011dynamic, pavone2010adaptive} is to divide the environment into equitable partitions. A robot is assigned to each partition thus simplifying the multiple-robot problem into several single-robot instances. In this paper we seek to improve multi-robot performance by making advances in the underlying single-robot policies.

\IEEEpubidadjcol

A characteristic of DVRPs is the load factor $\rho\in(0,1)$~\cite{bertsimas1991stochastic,smith2010dynamic,bajaj2019dynamic}. Letting $\lambda$ denote the arrival rate of incoming tasks and letting $\bar{s}$ be the expected service time, the load factor is given by $\rho=\lambda \bar{s}/m$ for $m$ robots. It captures the fraction of time the robot fleet \emph{must} be working to service tasks. In light load $\rho \rightarrow 0^+$, the queue length of unserviced tasks approaches zero, and robots have enough time between task arrivals to move to optimal waiting locations~\cite{bullo2011dynamic}. In heavy load $\rho \rightarrow 1^-$, all robots are continuously servicing tasks with no idle time, and the queue length of unserviced tasks approaches infinity. 
Neither of these states is operationally desirable: under light load, robots are underutilized, and under heavy load, task wait times become undesirably long~\cite{bertsimas1991stochastic}.
Previous work focused on guarantees under the extremes of light or heavy load, neglecting the space where these systems  operate ideally~\cite{whitt1992understanding}. 

\begin{figure}[t]
 \centering
 \begin{subfigure}{0.5\linewidth}
 \centering
 \includegraphics[width=0.8\linewidth]{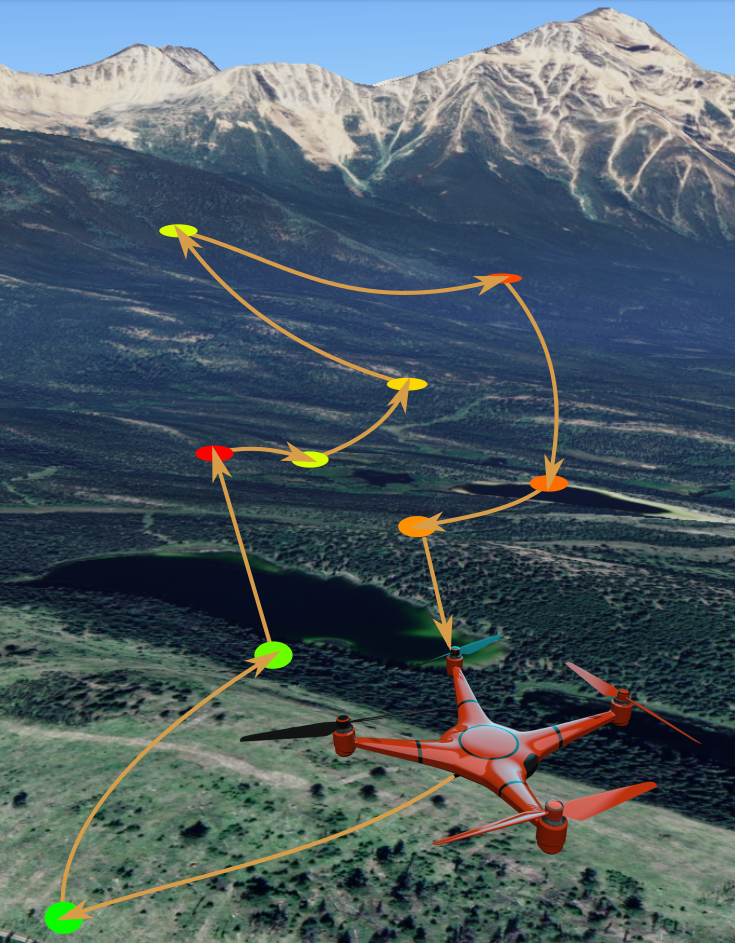}
 \caption{Minimize tour length}
 \label{fig:intro_tsp}
 \end{subfigure}%
 \begin{subfigure}{0.5\linewidth}
  \centering
  \includegraphics[width=0.8\linewidth]{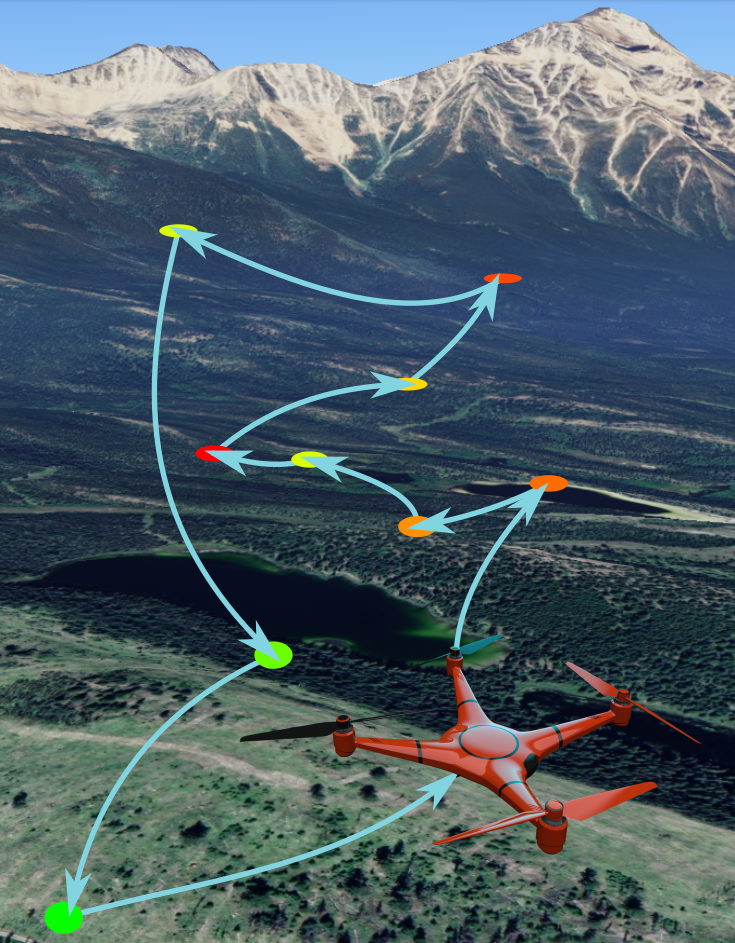}
  \caption{Minimize proposed cost}
  \label{fig:intro_ours}
 \end{subfigure}
 \caption{\footnotesize Effect of different cost functions on the optimal tour of a monitoring UAV. Triggered ground sensors add observation tasks (colored dots) for the UAV. Color (increasing values from green to red) represents how long a task has already been waiting when the tour is planned  (latent wait time).}
 \label{fig:intro_fig}
\end{figure}

In this paper, we focus on systems operating under moderate load,
where for example $\rho \in [.5,.9]$. Moderate load is arguably of the most practical significance, since it corresponds to the case where the queue of unserviced tasks is neither empty nor approaching infinite length.
In fact, as a rule of thumb, the load factor of a robot should be below $0.9$~\cite{whitt1992understanding} since the length of the outstanding task queue is proportional to $1/(1 - \rho)^2$~\cite{bertsimas1993stochastic}.
We propose a policy for moderate loads that seeks to solve the multi-objective problem of minimizing both average and maximum wait times.
To that end, we study a cost function to compute new tours where we minimize the sum of the wait times, each raised to an exponent $p > 1$. Under this cost function, tasks gain priority by waiting sufficiently long. While there is some advantage to this cost function under heavy load ($\rho \rightarrow 1^-$), the advantage disappears as the queue grows and travel distances between tasks approach zero. 
We illustrate the advantages of our method in Figure \ref{fig:intro_fig} for a monitoring application such as forest fire detection \cite{sudhakar2020unmanned}. A network of low-cost sensors (\textit{e.g.,} distributed in the environment, onboard high-altitude drones, or satellites) can detect points of interest. Upon detection, a drone with high-resolution sensors is tasked to visit these locations to collect more accurate data.
State-of-the-art methods~\cite{pavone2010adaptive} find tours of minimal length to minimize total mission time, shown in Figure~\ref{fig:intro_tsp}). In contrast, our method, shown in Figure~\ref{fig:intro_ours}), gives a higher priority to tasks that have waited longer and thus visits these locations earlier. Tasks with longer waiting times (red) are serviced earlier in the tour using our method, resulting in a lower mean service time for all tasks. This results in a lower average and maximum wait time.

\textbf{Related Work:}
Vehicle Routing problems have been studied in many forms over the past several decades~\cite{psaraftis1988dynamic,psaraftis1980dynamic, bertsimas1991stochastic}, with active research continuing today~\cite{rios2021recent, aksaray2016dynamic, vasile2017minimum, sarkar2018scalable, psaraftis2016dynamic, toth2014vehicle, soeffker2021stochastic, bopardikar2019dynamic}. Following the taxonomy in~\cite{psaraftis2016dynamic}, we consider the \emph{dynamic and stochastic} variant of DVRP where new tasks arrive online (dynamic) according to a stochastic process (stochastic) and are randomly distributed in the environment.

The categorization of light and heavy load states was considered in \cite{bertsimas1991stochastic}. The bounds on the performance of the optimal policy in heavy load given in~\cite{bertsimas1991stochastic, bertsimas1993stochastic} are a building block for a number of studies on DVRP variants~\cite{bullo2011dynamic, smith2010dynamic, pavone2010adaptive}. 
In~\cite{pavone2010adaptive}, the authors propose a $\Batch$ policy where the TSP is calculated on the task locations of a batch of tasks, and at each iteration, a randomly chosen fraction of the batch is serviced. The authors show that this approach is within a factor of two of optimal in heavy load.
The aforementioned studies focused on heavy or light loads and proposed policies that compute tours which minimize the total travel time (\emph{i.e.,} TSP minimization).
An alternative policy is to minimize the maximum wait time and/or the average wait time~\cite{campbell2008routing, huang2012models, kulich2017multi} in the context of fairness among tasks. 
Closely related to our approach, \cite{ferrucci2015general} uses a vehicle routing cost based on the square of wait times; 
while the cost function is a special case of our cost function where the wait time is raised to some exponent $p>1$, the policy is different. In \cite{ferrucci2015general} routes are recomputed whenever new tasks arrive, while our work uses the $\Batch$ approach from \cite{pavone2010adaptive}. This allows us to derive theoretical guarantees on system stability which are not provided in \cite{campbell2008routing, huang2012models, kulich2017multi,ferrucci2015general}.

The DVRP finds widespread applications in various robotics and transportation domains.
Fleets of autonomous vehicles are deployed for pickup-and-delivery throughout cities \cite{grippa2019drone, gao2021capacitated} or for providing on-site service in indoor environments such as hospitals \cite{sadeghi2018re, wilde2022online}. In many cases, vehicles need to react to new requests appearing over time.
For instance, in autonomous mobility on demand \cite{zardini2022analysis, zhang2016control, alonso2017demand} it is desired to minimize the average response time to the incoming requests, while also limiting maximum waiting times.
Other applications include the deployment of autonomous drones as mobile sensor networks \cite{bullo2011dynamic, pavone2010adaptive}, for surveillance \cite{ozkan2021uav, liu2019optimization}, search and rescue \cite{chandarana2018determining}, and in environmental monitoring~\cite{chandarana2021planning, sousa2020decentralized,smith2011persistent}.

\textbf{Contributions:}
Our contributions are as follows. First, we propose a new $\Batch$ policy using the sum of wait times, each raised to an exponent $p$, as the cost. This actively considers the time a task has already waited and thus allows for prioritizing long-waiting tasks. Second, we establish a relation between the length of the optimal tour for the proposed cost and the number of tasks in the tour, and prove that for $p> 1$ the proposed policy is stable under any load factor $\rho<1$, providing the first stability result for a policy with this type of cost function. Finally, we demonstrate that under moderate load
our policy achieves a lower average wait-time and reduces the number of outlying tasks with high wait-time.

\section{Problem Formulation}

In this section, we revisit the formulation of the well-known DVRP. Consider a convex and closed environment $\mathcal{E} \subset \mathbb{R}^d$ where $d$ is the dimension. A set of $m$ vehicles traveling with constant speed $v$ are servicing the tasks arriving in $\mathcal{E}$. 
The tasks arrive according to a Poisson process with time-intensity $\lambda \in \mathbb{R}_{>0}$. Task locations are distributed according to a spatial density $\varphi : \mathcal{E} \rightarrow \mathbb{R}_{>0}$. Task arrivals are i.i.d (independent and identically distributed), and servicing a task requires a vehicle to visit the location of the task. The time required to service task $j$ at its location, denoted by $s_j$, is i.i.d with finite first and second moments, \textit{i.e.,} $\overline{s}$ and $\overline{s^2}$. The \emph{load factor} for this system of $m$ vehicles is defined as $\rho = \lambda \overline{s}/m$. 
Given the arrived tasks and the current vehicle locations, a policy repeatedly computes tours for each vehicle.
The wait time of task $j$, denoted by $W_j$, is the time between the arrival time of the task and the time that a vehicle arrives at the task location and starts servicing it. The system time of task $j$, denoted by $T_j$, is the time between the arrival of the task and the time that a vehicle completes servicing the task, \textit{i.e.,} $T_j = W_j + s_j$. The steady-state system time is denoted by $\overline{T} = \lim \mathrm{sup}_{j \rightarrow \infty} \mathbb{E}[T_j]$. Let $\pi$ denote a tour of tasks found by following a routing policy, and let $\overline{T}_{\pi}$ be the corresponding system time. A policy is \emph{stable} if the expected number of outstanding tasks is uniformly bounded at all times. Finally, we assume that vehicles have sufficient capacity for all tasks and do not need to return to the depot.

\begin{problem}[DVRP]
\label{prob:DVRP}
Given an environment $\mathcal{E}$, $m$ vehicles, and tasks arriving according to a stochastic process with density $\lambda$, spatial distribution $\varphi$, and random service times following a Gaussian distribution $\mathcal{N}(\overline{s}, \sigma)$, find a stable policy that computes a tour $\pi$ minimizing the system time, \textit{i.e.,} $\mathrm{inf}_{\pi} \overline{T}_{\pi}$.
\end{problem}

\section{Approach}
We adapt a state-of-the-art partitioning approach~\cite{pavone2010adaptive} to cast the multi-robot setting of the DVRP into a set of single-robot instances.
After reviewing previously proposed DVRP policies with provable stability guarantees, we present a novel single vehicle policy based on a cost function where we take the $p$-norm of wait times and show that this policy is also stable.

\subsection{Vehicle Routing Cost Function}

We study how a single vehicle computes routes to service tasks in the environment $\mathcal{E} \subset \mathbb{R}^d$.
At a start time $t_0$, let $\T$ be set of $n$ tasks that have arrived but not been serviced. Further, let $x_s$ be the vehicle's starting pose, let $x_i, i=1,\dots,n$ denote the pose of task $\tau_i \in \T$, and let $L_{i,j}=||x_i - x_j||/v$ denote the travel time between $x_i, x_j$ for all $i,j\in\{1,\dots,n\}\cup\{0\}$. We assume constant speed $v$, which, without loss of generality, we set to 1 implying that length and travel time for any segment are equivalent. For each task $\tau_i$, let $t_i\geq 0$ denote the \emph{latent wait time} of task $\tau_i\in\T$: the difference between the time of planning the current tour and when $\tau_i$ arrived.  Finally, let $s_i$ denote the time required to service task $\tau_i$.  Thus, we decompose the system time, $T_i$, into three components: $T_i = t_i + \left\{\textit{tour travel time to $\tau_i$}\right\}_i + s_i$. 

Using these definitions, a tour $\pi$ is an ordering of task service visits that minimizes some undesirable properties. Given an arbitrary tour $\pi=(x_s,\tau_{1}, \tau_{2},\dots,\tau_{n})$ as well as a value $p\in\mathbb{N}_{\geq 1}$, we consider a cost function:
\begin{equation}
\label{eq:cp_cost}
 \begin{split}
  c^p(\pi) 
  &= \Bigg(\sum_{i=1}^n
  \bigg(
  \xoverbrace{ t_{i}}^{\mathclap{\text{latent wait time of task } \tau_i}} +\underbrace{\sum_{j=1}^i( \frac{L_{j, j-1}}{v} + \bar{s}}_{\mathclap{\text{time to reach and service task } \tau_i}} 
  )
  \bigg)^p
  \Bigg)^{1/p}.
 \end{split}
\end{equation}

We notice that $c^p$ represents the $L^p$-norm of the waiting time of each task in $\T$ assuming service times $s_i$ equal the expected value $\bar{s}$. We assume that the service time of any task is not revealed before servicing \cite{bullo2011dynamic}. Thus, treating service time as $\bar{s}$ when evaluating a candidate tour is not unreasonable. Since $\bar{s}$ is constant, in the remainder of this paper, we denote $\nicefrac{L_{j, j-1}}{v} + \bar{s}$ as simply $l_{j, j-1}$ for brevity. We offer an observation:

\begin{figure*}[t]
\centering
  \includegraphics[width = 0.95\linewidth]{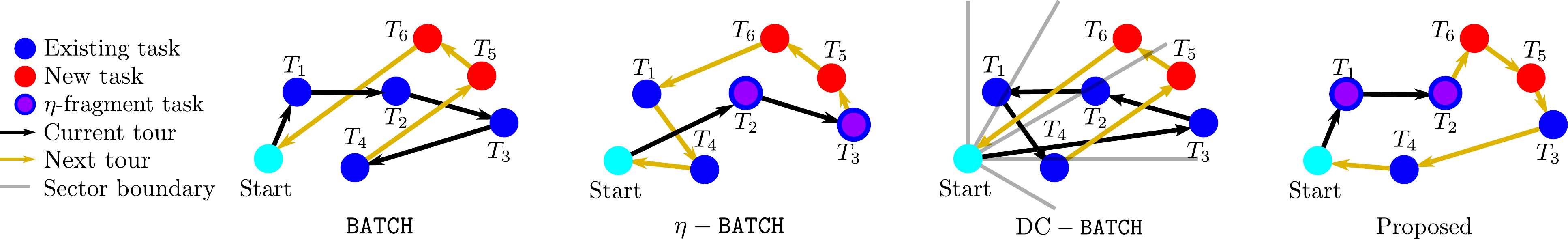}
  \caption{\footnotesize Example solutions for the different archetypes on the same tasks. Here, $x_s=\mathtt{Start}=x_{\text{waiting}}$. Tasks $\tau_1 - \tau_4$ arrived before planning the current tour, tasks $\tau_5, \tau_6$ arrive immediately after the tour commences execution.}
  \label{fig:Archetypes}
\end{figure*}

\begin{observation}[Variants of $c^p$ cost]
The cost $c^p(\pi)$ generalizes previously proposed cost functions:
\begin{enumerate}
 \item Minimizing $c^1(\pi)$ is equivalent to minimizing the mean wait time of all tasks in $\T$. Equation \eqref{eq:cp_cost} reduces to the sum of wait times and long waits are no longer penalized.
 
 \item Minimizing $c^{\infty}(\pi)$ is equivalent to minimizing the maximum wait time of all tasks in $\T$ since the longest wait time becomes the dominant term and is heavily penalized.
 If $t_{i}=0$ holds (\textit{i.e.,} all tasks arrive at the beginning of the tour), then 
 the problem is equivalent to minimizing the total travel distance \textit{i.e.,} solving the \emph{classic} TSP.
  
\end{enumerate}
\end{observation}

\subsection{Solution Archetypes and Proposed Approach}
\label{sec:SolutionArchs}
\begin{algorithm}[!t]
\small
\caption{Generalized Policy}
\label{alg:TOUR}
\begin{algorithmic}[1]
\Procedure{Tour}{$\T,p\geq 1, \eta\in(0,1], p_t\in\{0,1\}, \text{R}\in\{0,1\}$}
\State Replace $t_i$ with $p_tt_i$ in $c^p$ (equation \eqref{eq:cp_cost})
\While{$\T\neq \emptyset$}
\State Compute a $c^p$-minimizing tour $\pi^*$ of $\T$
\State ($x_s, x_{1}, x_{2},\dots,x_{n})\gets\pi^*$
\State index $\gets \big(\text{RandomInteger}(1, n)\big)^R$ 
\State Service tasks $\tau_{\text{index}},\dots, \tau_{\text{index}+\eta n}$ in order
\State Remove tasks $\tau_{\text{index}},\dots, \tau_{\text{index}+\eta n}$ from $\T$
\If{a new task $\tau_{n+1}$ arrives at any time}
\State add $\tau_{n+1}$ to $\T$
\EndIf
\EndWhile
\State return to $x_{\text{waiting}}$
\EndProcedure
\end{algorithmic}
\normalsize
\end{algorithm}

In this section, we present three state-of-the-art approaches to Problem \ref{prob:DVRP} as well as our proposed method. 
We begin with a generalized algorithm (Algorithm \ref{alg:TOUR}) that, with appropriately selected inputs, captures all four approaches.
The algorithm \verb|Tour| takes as input a set of $n$ tasks $\T$, an exponent $p$ for the $c^p$ cost defined in \eqref{eq:cp_cost}, a value $\eta\in(0,1]$, and two boolean variables $p_t$ and $R$. The high-level idea behind the algorithm is this: we start by computing a tour $\pi^*$ that minimizes $c^p$ over all tours of $\T$ (Line 4) and then select a fragment of length $\eta n$ --- which we refer to as an $\eta$-\emph{fragment} --- of this tour to service (Lines 6 - 8). The boolean input $R$ controls which fragment to service. If $R=0$ then index$=1$ (Line 6) and we service the first $\eta$-fragment, whereas if $R=1$ then we service a random $\eta$-fragment (Lines 7-8). If a new task $\tau_{n+1}$ arrives during this time, it is added to $\T$ (Lines 9 - 10). Once the $\eta$-fragment has been serviced, the process repeats. Finally, the boolean variable $p_t$ controls whether the latent wait time is considered in the cost $c^p$. If there are no outstanding tasks (Line 11), the robot returns to a given waiting pose $x_{\text{waiting}}$ usually defined as the centroid of the workspace \cite{bullo2011dynamic,pavone2010adaptive}.

\textbf{\texttt{BATCH} policy:}
The $\Batch$ policy proposed in~\cite{bullo2011dynamic} has the structure $\verb|TOUR|(\T, p=\infty, \eta=1, p_t=0, R=0)$.
Since $p=\infty$, and $p_t=0$, the cost $c^p$ is the total length of the tour, and $c^p$ is minimized by a tour $\TSP$ that solves the TSP. Further, since $\eta=1, R=0$, the full set of tasks $\T$ is serviced before re-planning. 

\textbf{$\boldsymbol{\eta}$-\texttt{BATCH} policy:}
In contrast, the authors of~\cite{pavone2010adaptive} proposed $\eta$-$\Batch$, which we capture with $\verb|TOUR|(\T, p=\infty, \eta\in(0,1), p_t=0, R=1)$. Again, the tour $\pi^*$ in Line 4 of Algorithm \ref{alg:TOUR} is $\TSP$ that solves the TSP for tasks $\T$. However, since $\eta\in(0,1)$, a $\eta$-fragment is randomly selected from $\pi^*$ and serviced before re-planning. 

The notion of a $\eta$-fragment is proposed in~\cite{pavone2010adaptive} to improve the average wait time of tasks in $\T$. If $\eta=1$, then new tasks wait for all existing tasks to be serviced before they are considered. However, if $\eta \in(0,1)$ and the first $\eta$-fragment is serviced before re-planning, then it is possible for some tasks to be continually overlooked. Thus, the $\eta$-$\Batch$ policy arbitrarily selects an $\eta$-fragment to service before re-planning ($R=1$) which ensures that, in expectation, no task must wait longer than $\nicefrac{1}{\eta}$ iterations of Algorithm \ref{alg:TOUR} before servicing.

\textbf{\texttt{DC}-\texttt{BATCH} policy:}
The single vehicle Divide and Conquer strategy~\cite{pavone2010adaptive} partitions the environment into $r$ equitable sectors. The algorithm plans a TSP tour for the tasks that arrive in a sector, services them, and then moves to the next sector with tasks in a round-robin fashion.

\textbf{Proposed $\boldsymbol{c^p}$-\texttt{BATCH} policy:}
We propose a new policy labeled $c^p$-$\Batch$ which has the form $\verb|TOUR|(\T, p\in(1, \infty), \eta\in(0, 1), p_t=1, R=0)$. That is, we include $t_i$ in the cost \eqref{eq:cp_cost} and thus do not find a tour $\pi^*$ of minimal length, but minimal $c^p$-cost. Further, we do not select the $\eta$-fragment randomly.

The motivation for our approach is that values $p\in(1, \infty)$ and $p_t=1$ incorporate the latent wait time of tasks into the cost. If tasks have been waiting too long to be serviced, it will become too costly not to service them. Thus we adopt the $\eta$-fragment strategy of~\cite{pavone2010adaptive} to improve performance without having to also impose randomness to mitigate long maximum wait times. 
The value $p=1$ has been excluded since it places no additional penalty on long waits, leading to the possibility of distant tasks being ignored~\cite{psaraftis1988dynamic}.
Examples of the $\Batch, \eta-\Batch$, $\DC$-$\Batch$, and proposed approaches are illustrated in Figure \ref{fig:Archetypes}.
Observe that $\eta-\Batch$ services  \blue{$\tau_5$} first: the randomness of the $\eta$-fragment causes this approach to miss \blue{$\tau_1$} which was closer to the start location. 
This does not occur in the proposed approach which services the first $\eta$-fragment before re-planning.  The stability of the proposed method is established next.

\section{Stability Analysis}
In this section, we prove the stability of the proposed approach. We use the following short-hand notation. Given a tour $\pi=(\tau_0,\dots,\tau_n)$ of a set of tasks $\T$, we use $l_j$ to represent the euclidean distance between tasks $\tau_{j-1},\tau_j\in \T$. That is, $l_j = l_{j, j-1}$. Further, $l(\pi)$ denotes the total length of $\pi$.
We start with assumptions:
\begin{assumption}
\label{ass:load}
Both $\bar{s}$ and $\lambda$ are known and the expected load factor is $\bar{\rho}=\lambda\bar{s}<1$.
\end{assumption}
\begin{assumption}
\label{ass:euclid}
All tasks arrive in a connected planar region with area $A$, finite perimeter $P$, and maximum Euclidean distance between tasks $Q'$.
\end{assumption}
\begin{assumption}
\label{ass:pre_wait}
At iteration 0 of Algorithm \ref{alg:TOUR}, the number of unserviced tasks is finite and  latent wait times $t_i$ are bounded by an arbitrarily large constant $R$ (\textit{i.e.,} $R = \max_{t_i \in \mathcal{T}} t_i$). 
\end{assumption}
Assumptions \ref{ass:load} and \ref{ass:euclid} are common in the literature (e.g.~\cite{bullo2011dynamic, pavone2010adaptive}). 
Further, Assumption \ref{ass:load} does not include the spatial distribution of tasks $\varphi$ and the variance of service times, $\overline{s^2}$ since the stability result does not depend on them.
Assumption \ref{ass:pre_wait} is reasonable since it states only existing tasks at the start have not been waiting for unbounded time. Letting $Q=Q'+\bar{s}$, we note that $Q$ represents the maximum time to service a task and move to the next under expected service times.

The proof of stability of the $c^p$-$\Batch$ policy follows closely the proofs for the policies proposed in~\cite{bullo2011dynamic},~\cite{pavone2010adaptive} in that it appeals to a well-established recursion. Letting $N_k, T_k$ be the random variables representing the number of outstanding tasks at the beginning of iteration $k$ of Algorithm \ref{alg:TOUR}, and the total service time of those tasks, respectively, and letting $\lambda, \eta$ be constants representing the arrival rate and fragment length (used as input to Algorithm \ref{alg:TOUR}), respectively, then from ~\cite{pavone2010adaptive}:
\begin{equation}
 \label{recursion}
 \E[N_{k+1}]=\underbrace{\lambda \eta \E[T_k]}_{\text{Newly arrived tasks}} + \underbrace{(1-\eta)\E[N_k]}_{\text{left-over tasks from iteration } k},
\end{equation}
The authors of~\cite{pavone2010adaptive} observe that at unit speed, $\E[T_k]\leq Q + \eta \E[l(\pi_k^*)] + \eta\E[N_k]\bar{s}$ and \eqref{recursion} reduces to
\begin{equation}
\label{eq:next}
 \begin{split}
  \E[N_{k+1}]\leq &\lambda Q+\lambda \eta \E[l(\pi^*_k)]+\rho\eta\E[N_k]\\
  &+ (1-\eta)\E[N_k].
 \end{split}
\end{equation}
At a high and informal level, \eqref{eq:next} implies that if $\E[l(\pi_k^*)]$ grows with \emph{less than linear} (LL) power in $\E[N_k]$, and $\rho<1$, then 
\begin{equation}
 \label{motivation}
 \lim_{\E[N_k]\rightarrow\infty}\left(\frac{\E[N_{k+1}]}{\E[N_k]}\right) \leq 1 - \eta(1 - \rho) < 1.
\end{equation}
implying (at a high level) the stability of the policy. In~\cite{bullo2011dynamic},~\cite{pavone2010adaptive}, the authors make use of the well established Beardwood-Halton-Hammersley (BHH) result~\cite{beardwood1959shortest} to prove that if $\pi^*_k=\TSP$, then $\E[l(\pi^*_k)]\leq K\sqrt{\E[N_k]}$ for a constant $K$ from which stability follows. Though the stability result in~\cite{pavone2010adaptive} is far more rigorous than the argument above, the crux of that argument, and ours, is the LL power of $\E[l(\pi^*_k)]$ in $\E[N_k]$. Unfortunately, the BHH result cannot be applied directly to a $c^p$-minimizing tour $\pi_{k}^*$ for general values of $p$, since the latent wait time $t_i, \tau_i\in\T$ cannot be separated from the cost. Therefore, unlike the tours $\TSP$ used in~\cite{bullo2011dynamic},~\cite{pavone2010adaptive}, the tours $\pi_k^*$ will depend on these latent wait times. As such, to follow the proofs of stability in~\cite{pavone2010adaptive}, we must establish the LL power of $\E[l(\pi_k^*)]$ in $\E[N_k]$ at each iteration $k$. We begin with supporting results.

\begin{lemma}[Bounding Tour Length By Cost]
\label{lem:LengthBound}
Given any $n$ tasks $\T$ and any tour $\pi$ on those tasks, if $p\geq 1$, then the length of $\pi$ obeys the following inequality:
\begin{equation}
 l(\pi)\leq \left(Q(p+1)c^p(\pi)^p\right)^{\frac{1}{p+1}}
\end{equation}
\end{lemma}
\begin{proof}
Given tour $\pi$ on tasks $\T$, we observe trivially that $c^p(\pi)^p\geq \sum_{i=1}^n\Big(\sum_{j=1}^i l_i\Big)^p $. Therefore, the length of any tour $l(\pi)$ is no more than the length of the maximum length tour $l_{\text{max}}(\T)$ subject to this constraint:
\begin{equation*}
 \begin{split}
  l(\pi) \leq l_{\text{max}}(\T)\doteq &\max_{l_1,\dots,l_n} l_1 + l_2 + \dots + l_n\\
  s.t. \ &\sum_{i=1}^n\Big(\sum_{j=1}^i l_i\Big)^p \leq c^p(\pi)^p, \ l_i \leq Q.
 \end{split}
\end{equation*}
It is not difficult to show that there exists a value $k$ such that $l_{\text{max}}(\T)$ has the form: $l_1=\dots= l_{k-1}=0$, and $l_k=\dots=l_n=Q$. That is, that the first $k$ lengths are 0 while the remainder take their maximum value. This is because increasing any value $l_i, i<k$ by $\delta$ would require a decrease of a value $l_j, j\geq k$ by more than $\delta$ in order to satisfy the constraint. To determine $k$, we replace the above general form of $l_{\text{max}}(\T)$ in the first constraint and simplify, resulting in $Q^p\sum_{i=1}^{n-k+1}i^p\leq c^p(\pi)^p$. Therefore, letting $L=n-k$, and observing that $\sum_{i=1}^{L+1}i^p \geq (L+1)^{p+1}(p+1)^{-1}$ for all $L\geq 0$ and $p\geq 1$, it must hold that
\begin{equation*}
 \begin{split}
Q^p(L+1)^{p+1}(p+1)^{-1}\leq c^p(\pi)^p\\
\implies L\leq \left(\frac{c^p(\pi)^p(p+1)}{Q^p}\right)^{1/p+1}.
 \end{split}
\end{equation*}
Finally, given the general form of $l_{\text{max}}(\T)$ described above, we observe that $l_{\text{max}}(\T)= LQ$ and the result follows.
\end{proof}
We now establish a bound on the optimal tour length on $n$ tasks $\T$. We let $t_{\text{max}}\geq \max_{\tau_i\in\T}t_i$ be the maximum latent wait time of tasks in $\T$. We assume that $t_{\text{max}}$ is bounded by a linear function of the number of tasks. 
\begin{lemma}[Deterministic bound: $l(\pi^*)$]
\label{lem:BatchStable}
For any $n$ tasks $\T$ and $p\geq 1$, if there exists constants $c_t\geq 0, \nu\in[0,1)$ such that $t_{\text{max}}\leq c_tn^{\nu}$ and $n\geq \nicefrac{P^2}{2A}$, then
\begin{equation}
\label{eq:lengthbound}
 l(\pi^*) \leq
 \underbrace{\Big(2^{p-1}Q(p+1)(c_t^p+2^p\sqrt{2A}^p)\Big)^{\frac{1}{p+1}}}_{\text{a constant}}n^{\frac{p\gamma + 1}{p+1}},
\end{equation}
where $\gamma=\max\{0.5, \nu\}$. 
\end{lemma}
\begin{proof}
 We begin by establishing an upper bound on $c^p(\TSP)^p$. Let $\tau_1, \tau_2, \dots, \tau_n$ denote the ordering of the tasks $\T$ as they appear in $\TSP$. Then, by the definition of the cost $c^p(\TSP)$, it must hold that
\begin{equation*}
\label{eq:BatchStable1}
 \begin{split}
  c^p(\TSP)^p = &\sum_{i=1}^n\Big(t_i + \sum_{j=1}^i l_j \Big)^p\leq n \Big(t_{\text{max}}+\sum_{j=1}^n l_j \Big)^p.
 \end{split}
\end{equation*}
Noting $\sum_{j=1}^nl_j=l(\TSP)$ and $l(\TSP)$ is bounded \cite{haimovich1985bounds} by $\sqrt{2A}\sqrt{n}+P$ for all $n$. Therefore, 
\begin{equation}
\label{eq:BatchStable2}
 \begin{split}
  &c^p(\TSP)^p\leq n\left( t_{\text{max}}+\sqrt{2An} + P\right)^p\\
  &\leq n\left(c_tn^{\nu} + 2\sqrt{2An}\right)^p, \ \text{since } t_{\text{max}}\leq c_tn^{\nu}, n\geq \nicefrac{P^2}{2A}\\
  &\leq 2^{p-1}(c_t^p+2^p\sqrt{2A}^p)n^{p\gamma + 1}.
 \end{split}
 \raisetag{1.2\normalbaselineskip}
\end{equation}
The final inequality holds since $(A+B)^p\leq 2^{p-1}(A^p+B^p)$ for all $A, B\geq 0, p\geq 1$ by the convexity of $x^p, x\geq 0$. Since $\pi^*$ minimizes $c^p$ over all tours of $\T$, we conclude that $c^p(\pi^*)^p\leq c^p(\TSP)^p$ and the result follows from \eqref{eq:BatchStable2} and Lemma \ref{lem:LengthBound}.
\end{proof}

Let $n_k$ denote the number of unserviced tasks at the beginning of iteration $k$ of Algorithm \ref{alg:TOUR}. Thus, $n_k$ is the realization of the random variable $N_k$.
For the remainder of our analysis we consider $n_k\geq P^2/2A$ for constants $P, A$. If there is no iteration $k$ such that $n_r\geq P^2/2A$ for every $r\geq k$, then stability holds trivially.  
Lemma \ref{lem:BatchStable} implies that if the latent wait times of tasks at iteration $k$ of Algorithm \ref{alg:TOUR} are all bounded by a function of $n_k$ with LL power then the same holds for the length of tour (since $\nicefrac{p\gamma + 1}{p+1}<1$). Using this result, we offer a deterministic bound on the length of the optimal tour at any iteration when $\eta=1$, and extend this result to $\eta\in(0, 1]$.

\begin{theorem}
\label{thm:lengthbounds}
On any iteration $k$ of Algorithm \ref{alg:TOUR}, if $\pi^*_k$ is the $c^p$-minimizing tour computed in Line 4, $p\geq 1$, $n_k\geq \nicefrac{P^2}{2A} $, and $\eta=1$ then there exists constants $\beta\geq 0, \kappa\in[0, 1)$ with
\begin{equation}
 \label{eq:sublin}
 l(\pi^*_k)\leq \beta n_k^{\kappa}.
\end{equation}
\end{theorem}
\begin{proof}
We offer a proof by way of induction on the number of iterations $k$ of Algorithm \ref{alg:TOUR}. Let $\T_k$ denote the outstanding tasks at the start of iteration $k$, and $\pi_k^*$ the tour of $\T_k$ from Line 4. In the base case, $k=1$ and by Assumption \ref{ass:pre_wait}, $t_{\max}\leq R$. Therefore, $t_{\text{max}}\leq c_tn^{\nu}$ with $c_t=R, \nu=0$, and the result holds in the base case by Lemma \ref{lem:BatchStable}, since $\gamma=0.5$ and $\nicefrac{p+2}{2(p+1)}<1$ for $p\geq 1$. 

The induction assumption is that at iteration $k$, $l(\pi_k^*)\leq \beta n_k^{\kappa}, \beta\geq 0, \kappa\in[0,1)$. Since $\eta=1$, the maximum latent wait time of any task in $\T_{k+1}$ is $l(\pi^*_k)$ which occurs if a task arrives just after the beginning of iteration $k$. If $n_{k+1}< n_k$, then the result holds by the induction assumption. Otherwise, $n_{k+1}\geq n_k$ and by the induction assumption $t_{\max}\leq l(\pi^*_k)\leq \beta n_k^{\kappa}\leq\beta n_{k+1}^{\kappa}$. Thus the conditions of Lemma \ref{lem:BatchStable} hold with $c_t=\beta, \nu=\kappa$, and by \eqref{eq:lengthbound}, $l(\pi^*_{k+1})\leq \bar{C}n_{k+1}^{\nicefrac{p\gamma + 1}{p+1}}$ where $\bar{C}$ is the constant coefficient in \eqref{eq:lengthbound} with $\gamma=\max\{\nicefrac{1}{2}, \kappa\}<1$. Since $\gamma<1$, it must hold that $\nicefrac{p\gamma + 1}{p+1} < 1$ and the result holds at iteration $k+1$ concluding the proof.
\end{proof}

\begin{theorem}
\label{prop:lengthboundsSmalleta}
The result of Theorem \ref{thm:lengthbounds} still holds if $\eta\in(0, 1], p\geq 1$ for sufficiently large $n_k$.
\end{theorem}
\begin{proof}
 We begin by proving a sub-claim: given a set of $n$ tasks $\T_n$, and a new task $\tau_{n+1}\notin \T_n$, let $\T_{n+1}=\T_n\cup\{\tau_{n+1}\}$ and $\pi^*_n, \pi^*_{n+1}$ denote $c^p$-minimizing tours of $\T_n, \T_{n+1}$, respectively. If $n$ is sufficiently large and there exists constants $c_l, c_u\geq 0, \nu\in[0,1)$ such that $c_l\sqrt{n}\leq t_i\leq c_un^{\nu}$ for all $\tau_i\in\T_n$, then for any $\eta\in(0,1]$, $\tau_{n+1}$ will not appear in the first $\eta$-fragment of $\pi_{n+1}^*$. To prove this sub-claim, we show that any tour $\pi_{n+1}'$ of $\T_{n+1}$ with $\tau_{n+1}$ in the first $\eta$-fragment cannot be optimal by proposing a lower cost tour $\pi_{n+1}$. Let
\begin{equation*}
 \begin{split}
 \pi_n^* &= (\tau_1, \tau_2,\dots, \tau_n)\\
 \pi_{n+1} &= (\tau_1, \tau_2,\dots, \tau_n, \tau_{n+1})\\
\pi_{n+1}' &= (\tau_{i_1},\dots,\tau_{i_r}, \tau_{n+1}, \tau_{i_{r+1}}, \dots, \tau_{i_n})\\
\pi_n' &= (\tau_{i_1},\tau_{i_2},\dots,\tau_{i_n}).
 \end{split}
\end{equation*}
Here, $\pi_n^*$ denotes the tour that minimizes $c^p$ for tasks $\T_n$, $\pi_{n+1}$ is the tour of tasks $\T_{n+1}$ that services tasks according to $\pi_n^*$ and then services $\tau_{n+1}$, $\pi'_{n+1}$ is any tour of tasks $\T_{n+1}$ with $\tau_{n+1}$ inserted after $r\leq \eta n$ tasks, and $\pi'_n$ is the tour that is identical to $\pi'_{n+1}$ but with task $\tau_{n+1}$ removed. To prove the sub-claim, we will show that for large $n$, $c^p(\pi_{n+1}) < c^p(\pi_{n+1}')$. To this end, let $w_i(\pi)$ denote the wait time of task $\tau_i$ in the tour $\pi$, and let $\Delta$ the additional wait time of all tasks $\tau_{j}, j\geq i_{r+1}$ incurred from servicing $\tau_{n+1}$ first in $\pi'_{n+1}$. If $\Delta=0$ the claim is redundant to the proof as new tasks can be serviced instantaneously without increasing the wait time of any other task. Otherwise, $\Delta > 0$, and 
\begin{equation*}
 \begin{split}
&c^p(\pi'_{n+1})^p\geq\sum_{j=i_1}^{i_r}w_j^p+\sum_{j=i_{r+1}}^{i_n} (w_i+\Delta)^p \\
&\geq \sum_{j=i_1}^{i_n}w_j^p + p\Delta \sum_{j=i_{r+1}}^{i_n}w_i^{p-1} \geq c^p(\pi_n^*)^p + p\Delta (n-r)n^{\frac{p-1}{2}}.
 \end{split}
\end{equation*}
The above uses the fact that $(A+B)^p\geq A^p+pBA^{p-1}$ for all $A, B\geq 0, p\geq 1$ and $w_i\geq t_i\geq \sqrt{n}$. Further, observe that $c^p(\pi_{n+1})^p\leq c^p(\pi^*_n)^p + (l(\pi^*_n)+Q)^p$. Therefore, since $r\leq \eta n$, and $t_i\leq c_u n^{\nu}$, it follows by Lemma \ref{lem:BatchStable} that
\begin{equation*}
 \begin{split}
  &c^p(\pi'_{n+1})^p-c^p(\pi'_{n+1})^p\geq p\Delta (1-\eta)n^{\frac{p+1}{2}}-(l(\pi^*_n)+Q)^p\\
  &\geq p\Delta (1-\eta)n^{\frac{p+1}{2}} - (\bar{C}n^{\frac{p\nu + 1}{2(p+1)}}+Q)^p.
 \end{split} 
\end{equation*}
Finally, noting that $\nicefrac{p+1}{2}> \nicefrac{p(p\nu + 1)}{2(p+1)}$, we have that $c^p(\pi'_{n+1})^p-c^p(\pi'_{n+1})^p\geq 0$ for large $n$ establishing the sub-claim. To prove the result of the Theorem, it suffices to show that at every iteration of Algorithm \ref{alg:TOUR}, $t_{\max}\leq c_t n_k^{\nu}$ for some $c_t\geq 0, \nu\in[0,1)$ since the result of the Theorem would follow in an identical manner to the proof of Theorem \ref{thm:lengthbounds}. The proof that $t_{\max}\leq c_t n_k^{\nu}$ at every iteration $k$ is largely omitted for brevity, but we describe the structure. 
By strong induction on the number of iterations, suppose that $n$ tasks $\Told$ at iteration $k$ arrived before the start of iteration $k-1$ but were not serviced in that iteration. Then their wait time is at least $\eta l(\pi^*_{k-1})\geq \eta l(\TSP_k)\geq \eta \sqrt{n}$. Further, by the strong induction assumption, $t_{\max}\leq c_t n^{\nu}$. Therefore, the conditions of the sub-claim hold implying that the first $\eta$-fragment is comprised entirely of these tasks. 
This continues until at most $\nicefrac{1}{\eta}$ iterations when all the tasks in $\Told$ are serviced. Indeed, in the worst case, if no tasks in $\Told$ are serviced in $\nicefrac{1}{\eta}$ iterations, we can show using the above analysis that $\Told$ occupies the first $n=\eta (\nicefrac{n}{\eta})$-fragment. On each iteration $i$ before the tasks in $\Told$ have all been serviced, they remain the oldest tasks in $\T_i$, and can be shown to have a maximum latent waiting time $t_{\max}\leq \beta n^{\nu}/\eta^{\nu}$ using a similar argument as Theorem \ref{thm:lengthbounds}. 
\end{proof}
Finally, we prove the stability of the proposed policy.

\begin{theorem}[Stability of Proposed Algorithm]
For all load factors $\rho <1$, the proposed policy $c^p- \Batch$ with $p\geq 1, \eta > 0$ is stable. 
\label{thm:stability}
\end{theorem}
\begin{proof}
By Theorem \ref{prop:lengthboundsSmalleta}, for sufficiently large $n_k$, it holds that $l(\pi^*_k)\leq \beta n_k^{\kappa}$ for constants $\beta\geq0, \kappa \in[0,1)$. Therefore, $\E[l(\pi^*_k)]\leq \beta \E[N_k]^{\kappa}$ by Jensen's inequality. The remainder of the proof is nearly identical to~\cite[Theorem 5.1]{pavone2010adaptive} with $\E[l(\pi^*_k)]\leq K\sqrt{\E[N_k]}$ replaced with $\E[l(\pi^*_k)]\leq \beta\E[N_k]^{\kappa}$. From \eqref{eq:next} and Theorem \ref{thm:lengthbounds}, 
\begin{equation}
\label{eq:next2}
 \begin{split}
  \E[N_{k+1}]\leq &\lambda Q+\lambda \eta \beta \E[N_k]^{\kappa}+\rho\eta\E[N_k]\\
  &+ (1-\eta)\E[N_k],
 \end{split}
\end{equation}
In an identical manner to~\cite[Theorem 5.1]{pavone2010adaptive}, it can be shown that $\E[N_{k+1}]\rightarrow_{k}\E[N_k]$ which may be substituted in \eqref{eq:next} to produce a closed form constant bound on $\E[N_k]$.
\end{proof}

In essence, Theorem \ref{thm:stability} ensures that the number of unserviced tasks remains bounded at all times, \textit{i.e.,} the policy is stable, and thus the system does not become overburdened.


\section{Simulation Results}
\label{sec:results}

We demonstrate the performance of our proposed $c^p-\Batch$ method in a series of numerical experiments, comparing it against several state-of-the-art baselines.

\textbf{Experiment setup: }As our primary experiment, we consider the single vehicle DVRP in the Euclidean plane, consistent with the problem formulation and previous studies \cite{bullo2011dynamic, pavone2010adaptive}.
In a second experiment, we examine the multi-robot case and use real-world data on a roadmap, \textit{i.e.,} a non-euclidean, non-convex, non-symmetric environment. The stability results established herein for the proposed method, and in \cite{bullo2011dynamic, pavone2010adaptive} for the $\mathtt{BATCH}$, $\eta$-$\mathtt{BATCH}$ methods do not apply under this setting since Assumption \ref{ass:euclid} is violated.
All tours were computed using the LKH-3 solver~\cite{helsgaun2017extension}.

\begin{figure}[t]
  \centering
  \includegraphics[width=0.95\linewidth]{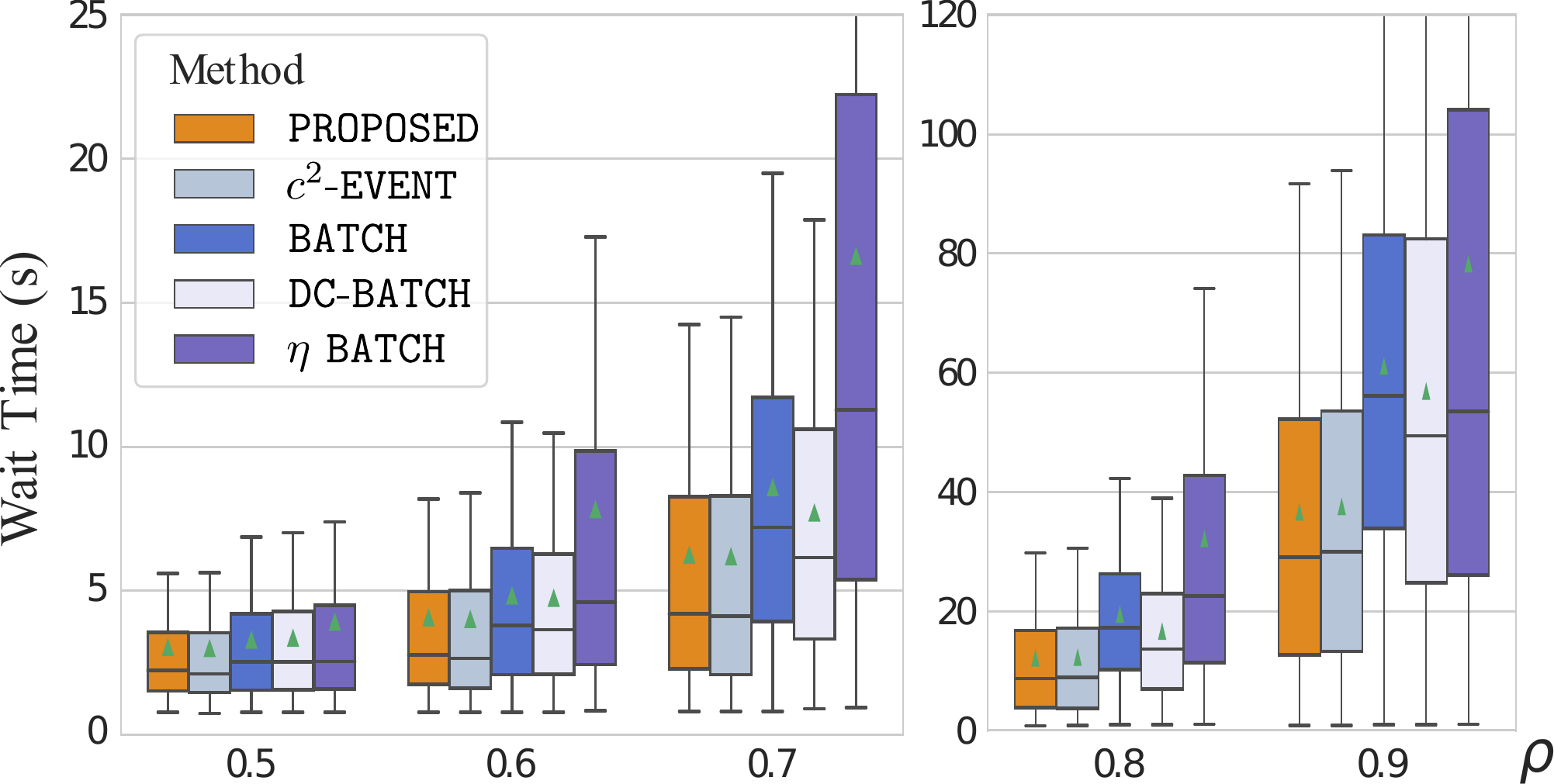}
   \caption{\footnotesize Experiment 1: Comparison against baselines in the Euclidean environment, showing wait times as a function of the load factor $\rho$. Boxes and whiskers show the four quartiles over all trials, mean wait times are indicated by the green triangle.}
 \label{fig:experiment1}
\end{figure}

\textbf{Baseline methods: }
We compare our approach against several DVRP baselines: $\Batch$~\cite{bullo2011dynamic}, $\eta$-$\Batch$~\cite{pavone2010adaptive}, and $\DC$-$\Batch$~\cite{pavone2010adaptive} discussed in Section \ref{sec:SolutionArchs}. 
 For $\eta$-$\Batch$ we use $\eta=0.2$, and for $\DC$-$\Batch$ we use $r=10$ sectors, consistent with~\cite{pavone2010adaptive}.
Finally, we also consider the event-triggered re-planning policy where the cost function is the quadratic wait times (\emph{i.e.,} $p=2$ in \eqref{eq:cp_cost}) from \cite{ferrucci2015general}, labelled as $c^2$-$\Cont$.

 \subsection{Experiment 1 -- Single robot in the Euclidean plane}
The first simulation environment is a unit square with 3000 tasks and a uniform arrival distribution. Service times are modeled as a Gaussian with a mean of $1$ and a standard deviation of $0.1$. We begin by considering a single robot moving at speed $v = 1$. We notice that since the multi-robot case is solved by partitioning the environment and then solving a single-robot DVRP in each partition, the results generalize trivially to multiple robots \cite{pavone2010adaptive}.

\paragraph{Algorithm parameters}
The proposed method has two hyper-parameters: the cost exponent $p$ and the batch fragment $\eta$. 
In general, smaller values of $\eta$ allow for more frequent re-planning and thus better performance; the stability  result still applies provided $\eta>0$.
We choose $\eta=0.05$, which is significantly lower than the proposed value of $0.2$ from \cite{pavone2010adaptive}. Empirically, we found that a value of $p = 1.5$ provides the best performance in terms of mean, median, and as well as number of outliers.
Thus, we use $\eta=0.05$ and $p=1.5$ for the proposed approach, simply denoted as $\mathtt{PROPOSED}$.

\paragraph{Comparison against baselines}

Figure \ref{fig:experiment1} illustrates the resulting wait times for the proposed approach and the baselines.  
As the load factor increases towards $\rho =0.9$, all baseline $\Batch$ methods show an increase in mean and variance compared to the $\mathtt{PROPOSED}$. 
The numerical results, found in Table~\ref{table:task-time-data}, additionally include a method denoted $\mathtt{PROPOSED}~\eta=0.2$, identical to $\mathtt{PROPOSED}$ (\textit{i.e.,} $p=1.5$), but with $\eta=0.2$ as in $\eta$-$\Batch$ from \cite{pavone2010adaptive}. We include this method to explicitly illustrate that any improved performance of $\mathtt{PROPOSED}$ over $\eta-\Batch$ is not simply due to the fact that $\mathtt{PROPOSED}$ features a smaller value of $\eta$. We observe that under loads from $\rho=.5$ and $\rho=.6$, the different methods perform relatively similarly, only $\eta-\Batch$ shows a higher mean and upper end of the distributions.
The method $\mathtt{PROPOSED}$ and $\mathtt{PROPOSED} \ \eta=0.2$ provide a consistently better performance than all baselines with exception of $c^2$-$\Cont$. While the smaller $\eta=0.05$ used in $\mathtt{PROPOSED}$ gives an additional performance boost, the principal improvement comes from the proposed cost function.
The mean wait times averaged over all workloads are increased by factors $1.39$ for $\Batch$, $2.14$ for $\eta$-$\Batch$, $1.28$ for $\mathtt{DC}$-$\Batch$, and $1.002$ for $c^2$-$\Cont$, compared to $\mathtt{PROPOSED}$.
Further, all approaches exhibit large deviations from the mean. However, $\mathtt{PROPOSED}$ is always among the best with respect to variances, medians and $75^{\text{th}}\%$.

In summary, the proposed method shows substantial improvements in mean and maximum wait times compared to all other $\Batch$ methods. Moreover, $c^2$-$\Cont$ -- which possesses no stability guarantees -- only achieves the same performance as the proposed approach, despite re-planning more frequently.

\begin{figure}[t]
 \centering
 \begin{subfigure}[b]{0.45\linewidth}
  \centering
  \includegraphics[width=0.95\linewidth]{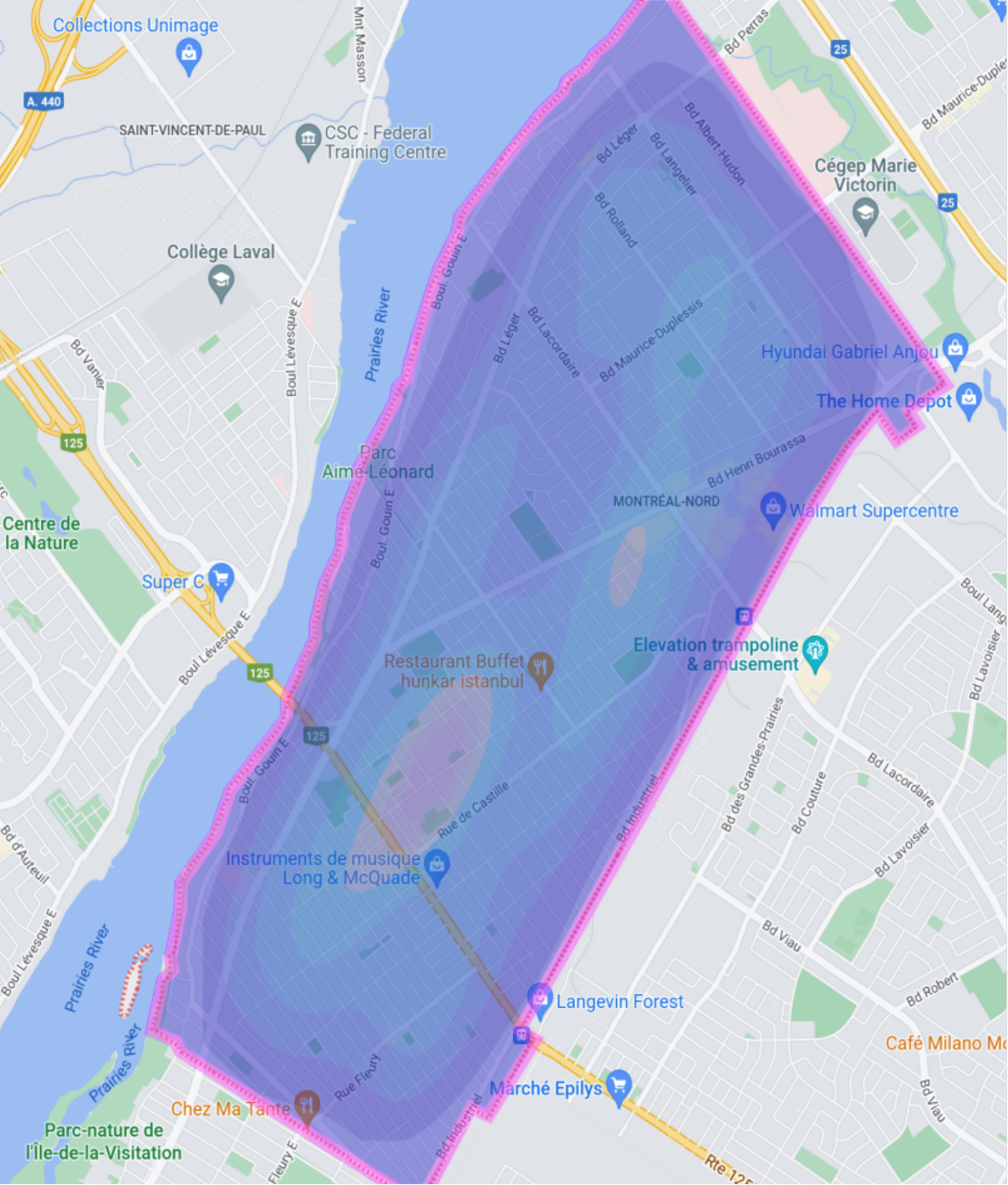}
  \caption{Borough of Montr\'eal-Nord.}
  \label{fig:montreal_map}
 \end{subfigure}%
 \hfill
 \begin{subfigure}[b]{0.55\linewidth}
  \centering
  \includegraphics[width=0.95\linewidth]{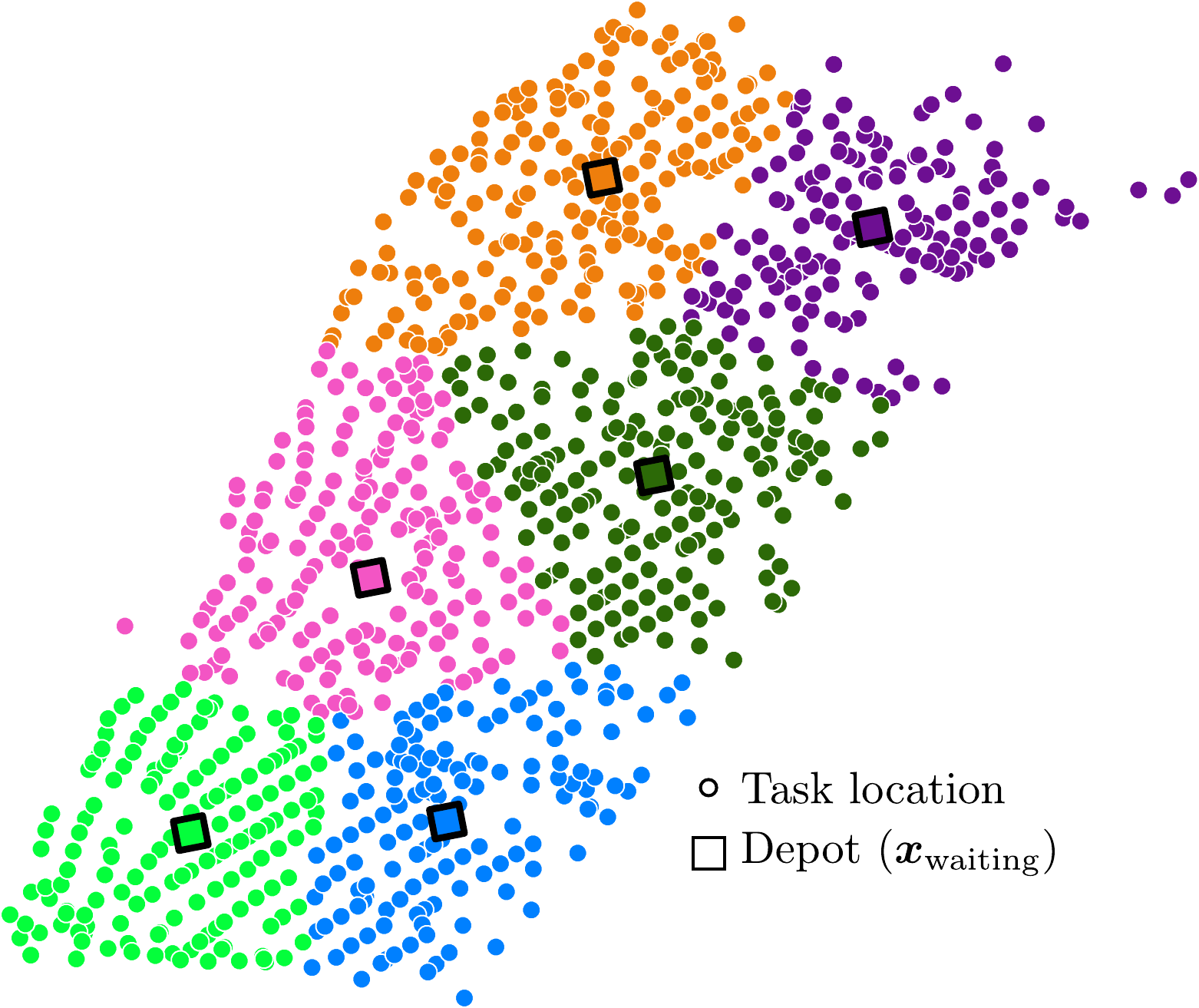}
  \caption{Partitioning for multi-robot DVR.}
  \label{fig:sim_montreal_clusters}
 \end{subfigure}
 \caption{\footnotesize Non-emergency assistance requests in the borough of Montr\'eal-Nord, Montr\'eal, Canada, between 2017-2019~\cite{montreal2016opendata}. (a): Map with the request density expressed as a heat map. (b): Partitions and depot locations.}
 \label{fig:montreal}
\end{figure}

\begin{table*}[t]
\caption{Results for Experiment 1. We show mean standard deviation ($\sigma$) and 95-percentile of {wait times} in \textit{seconds}.}
\label{table:task-time-data}
\begin{center}
\begin{tabular}{@{} l c c  c c  c c  c c  c c  @{}}
\toprule
 & \multicolumn{2}{c}{$\rho=0.5$} & \multicolumn{2}{c}{$\rho=0.6$} & \multicolumn{2}{c}{$\rho=0.7$} & \multicolumn{2}{c}{$\rho=0.8$} & \multicolumn{2}{c}{$\rho=0.9$}  \\
Method  & Mean\rpm $\sigma$ & 95\%  & Mean\rpm $\sigma$ & 95\%  & Mean\rpm $\sigma$ & 95\%  & Mean\rpm $\sigma$ & 95\%  & Mean\rpm $\sigma$ & 95\%  \\
\midrule
$\mathtt{PROPOSED}$ &   \textbf{3.0} \rpm   2.4 &   \textbf{7.7} &   4.1 \rpm   3.5 &  \textbf{11.2} &   \textbf{6.2} \rpm   5.7 &  \textbf{17.8} &  \textbf{12.1} \rpm  11.0 &  \textbf{33.8} &  \textbf{36.5} \rpm  31.0 &  \textbf{96.2}\\
$\mathtt{PROPOSED}$ $\eta$=$0.2$ &   \textbf{3.0}\rpm   2.4 &   \textbf{7.7} &   4.1 \rpm  {3.5} &  \textbf{11.2} &   6.3 \rpm   {5.7} &  17.9 &  12.4 \rpm  11.2 &  34.4 &  39.2 \rpm  31.8 &  99.8\\
\midrule
$c^2$-$\mathtt{EVENT}$ &  \textbf{ 3.0} \rpm   2.4 &   7.9 &   \textbf{4.0} \rpm   3.6 &  11.3 &   \textbf{6.2} \rpm   5.8 &  17.9 &  12.2 \rpm  11.2 &  34.8 &  37.5 \rpm  31.4 &  98.1\\
$\mathtt{BATCH}$ &   3.3 \rpm   2.3 &   8.0 &   4.8 \rpm   3.5 &  11.9 &   8.6 \rpm   6.0 &  20.3 &  19.4 \rpm  12.1 &  42.7 &  61.0 \rpm  35.5 & 127.4\\
$\mathtt{DC}$-$\mathtt{BATCH}$ &   3.3 \rpm   2.4 &   8.3 &   4.7 \rpm   3.5 &  12.0 &   7.7 \rpm   5.7 &  19.1 &  16.6 \rpm  12.2 &  41.2 &  56.8 \rpm  39.0 & 130.5\\
$\eta$-$\mathtt{BATCH}$ &   3.9 \rpm   3.9 &  11.5 &   7.8 \rpm   8.6 &  24.9 &  16.6 \rpm  16.4 &  49.4 &  32.1 \rpm  30.5 &  92.1 &  78.2 \rpm  76.7 & 229.3\\
\bottomrule
\end{tabular}
\end{center}
\end{table*}
\subsection{Experiment 2 -- Multiple Robots with real-world data set} 

As a further illustration of the applicability of our method, we simulate task arrivals for an on-site service technician based on historical 3--1--1 calls in Montreal-Nord, a borough of the city of Montreal, Canada~\cite{montreal2016opendata} (see Figure \ref{fig:montreal}). For this experiment, we deploy a fleet of $m=6$ robots to service $5000$ tasks. Thus, we divide the environment into 6 partitions using a $k$-nearest clustering of the request locations, shown in Figure \ref{fig:sim_montreal_clusters}.
Despite not guaranteeing equitable partitions, $k$-nearest methods are widely used in practise.
%
All map data taken from OpenStreetMaps\footnote{Map data copyrighted OpenStreetMap contributors and available from \url{https://www.openstreetmap.org}}. The average travel time over the entire map is $296s$, service times are drawn from a normal distribution with a mean of $10$min and variance of $3$min.
We used an overall load factor of $\rho=.74$, selected to guarantee $\rho < 1$ within each partition.

Results are shown in Figure \ref{fig:sim_montreal_baselines} and Table \ref{table:montreal-task-time-data}. 
We observe that since the used partitioning technique is not equitable the wait times vary significantly between partitions.
In fact, the load in partitions 1 and 2 is at the lower end of moderate load ($\approx .6$), while partition 6 reaches a single-robot load factor of $.99$. Nonetheless, for all partitions the proposed method achieves the lowest mean, and with exception of region 1 and 6 the lowest $95th$ percentile. 
Mean wait times are reduced by $20\%$ compared to the best baseline ($\mathtt{DC}$-$\mathtt{BATCH}$), \textit{i.e.,} $30min$. Moreover, we notice that the advantage of our method increases for the partitions under higher loads.
In summary, we have shown the proposed method also effectively improves task wait times in a real-world multi-robot setting.

\begin{figure}
 \centering
 \includegraphics[width=0.9\linewidth]{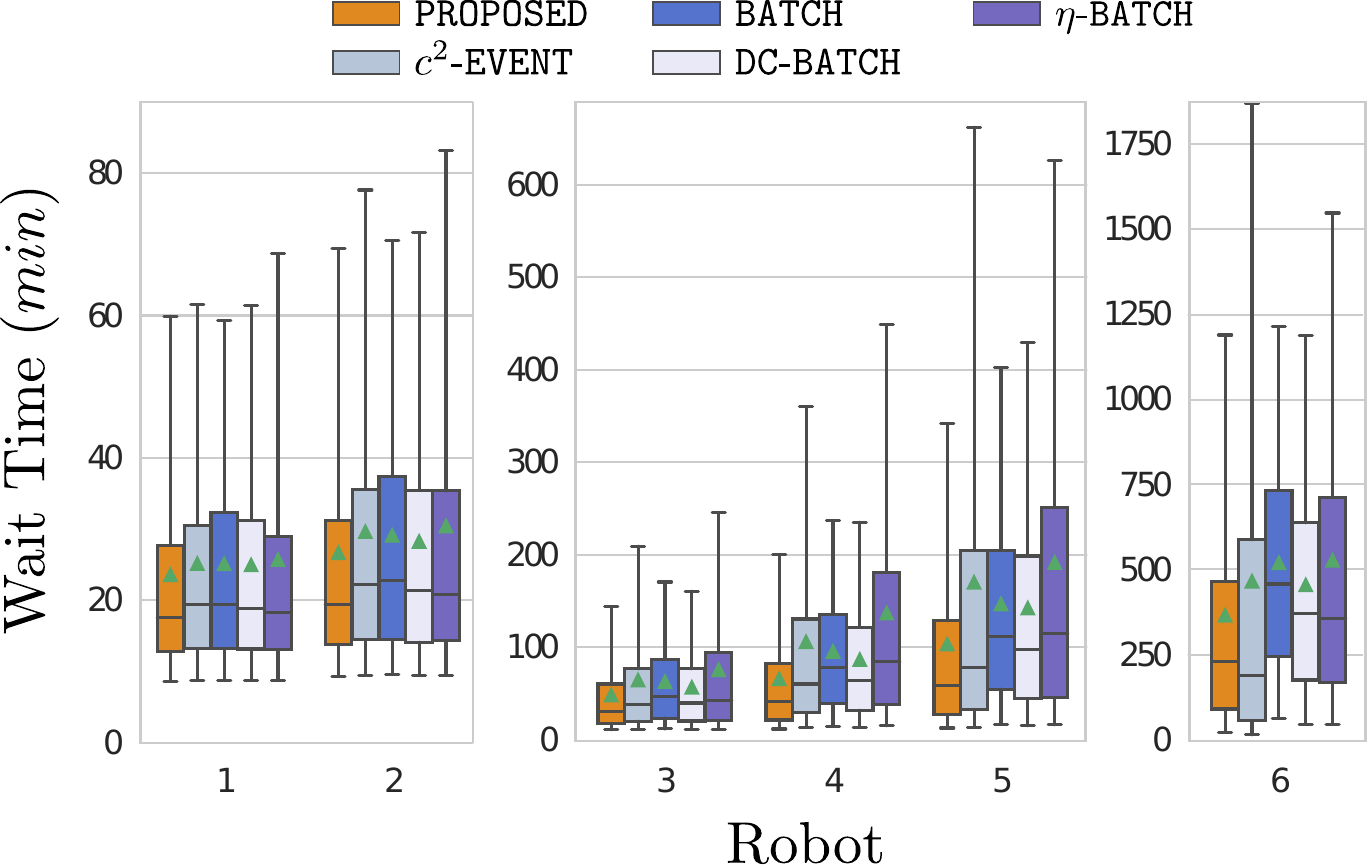}
 \caption{\footnotesize Experiment 2: Comparison against baselines. Wait times for a load factor $\rho=0.74$, separated by partition. Boxes and whiskers show the four quartiles over all trials, mean wait times are indicated by the green triangle.}
 \label{fig:sim_montreal_baselines}
\end{figure}

%
\begin{table*}
\caption{Results for Experiment 2, separated by region. We show mean, standard deviation $(\sigma)$ and 95-percentile of {wait times} in \textit{minutes}.}
\label{table:montreal-task-time-data}
\begin{center}
\begin{tabular}{@{} l c c  c c  c c  c c  c c  c c  c c  @{}}
\toprule
 & \multicolumn{2}{c}{Robot 1} & \multicolumn{2}{c}{Robot 2} & \multicolumn{2}{c}{Robot 3} & \multicolumn{2}{c}{Robot 4} & \multicolumn{2}{c}{Robot 5} & \multicolumn{2}{c}{Robot 6} & \multicolumn{2}{c}{All} \\
Method  & Mean\rpm $\sigma$ & 95\%  & Mean\rpm $\sigma$ & 95\%  & Mean\rpm $\sigma$ & 95\%  & Mean\rpm $\sigma$ & 95\%  & Mean\rpm $\sigma$ & 95\%  & Mean\rpm $\sigma$ & 95\%  & Mean\rpm $\sigma$ & 95\%  \\
\midrule
$\mathtt{PROPOSED}$ & \textbf{24}\rpm18 & 60 & \textbf{27}\rpm22 & \textbf{69} & \textbf{48}\rpm49 & \textbf{144} & \textbf{66}\rpm70 & \textbf{201} & \textbf{103}\rpm127 & \textbf{342} & \textbf{367}\rpm450 & 1190& \textbf{122}\rpm250 & \textbf{503}\\
\midrule
$c^2$-$\mathtt{EVENT}$ & 25\rpm19 & 62 & 30\rpm24 & 78 & 65\rpm76 & 209 & 106\rpm123 & 361 & 170\rpm237 & 663 & 467\rpm680 & 1868& 161\rpm361 & 715\\
$\mathtt{BATCH}$ & 25\rpm17 & \textbf{59} & 29\rpm20 & 70 & 63\rpm52 & 171 & 95\rpm71 & 237 & 147\rpm123 & 403 & 520\rpm351 & 1214& 168\rpm250 & 736\\
$\mathtt{DC}$-$\mathtt{BATCH}$ & 25\rpm18 & 62 & 28\rpm21 & 72 & 57\rpm50 & 161 & 87\rpm71 & 236 & 142\rpm132 & 429 & 456\rpm359 & \textbf{1187}& 152\rpm237 & 650\\
$\eta$-$\mathtt{BATCH}$ & 26\rpm22 & 69 & 30\rpm28 & 83 & 75\rpm89 & 246 & 137\rpm150 & 449 & 192\rpm224 & 628 & 528\rpm531 & 1553& 187\rpm323 & 792\\
\bottomrule
\end{tabular}
\end{center}
\end{table*}

\section{Conclusion}
We revisited the classic Dynamic Vehicle Routing Problem for stochastic arrivals. We proposed a new $\Batch$ policy that seeks to minimize both average and maximum wait times and showed that this policy is stable even under heavy loads. 
In simulations we showed that the proposed method outperforms existing baseline policies under various moderate load settings ($\rho\in[0.5, 0.9]$). The baseline of event-triggered re-planning, $c^2$-$\Cont$, with a quadratic cost shows a performance comparable to ours in the Euclidean case, yet this method is computationally more burdensome and does not come with theoretical guarantees on stability.
Future work should investigate the applicability of the proposed $p$-norm cost for other variants of DVR, such as pick-up-and-delivery where the problem cannot be cast into several single-robot instances.

\bibliographystyle{IEEEtran}
\bibliography{references}

\end{document}